\documentclass[runningheads]{llncs}

\usepackage{amsmath}
\usepackage{amssymb}
\usepackage{bbm}
\usepackage{algorithm}
\usepackage{tikz}
\usepackage{pgfplots}
\usepackage{algpseudocode}
\usepackage{caption}
\usepackage{subcaption}
\usepackage{multirow}

\usepackage[T1]{fontenc}
\usepackage{graphicx}
\usepackage{hyperref}
\usepackage{color}

\begin{document}
\title{PDD-SHAP: Fast Approximations for Shapley Values using Functional Decomposition}
\titlerunning{PDD-SHAP}
%
\author{Arne Gevaert\inst{1}\orcidID{0000-0003-4130-8151} \and
Yvan Saeys\inst{1}\orcidID{0000-0002-0415-1506}}
\authorrunning{A. Gevaert, Y. Saeys}
%
\institute{Department of Applied Mathematics and Statistics\\Ghent University, Ghent, Belgium}
\maketitle              
\begin{abstract}
Because of their strong theoretical properties, Shapley values have become very popular as a way to explain predictions made by black box models. Unfortuately, most existing techniques to compute Shapley values are computationally very expensive. We propose PDD-SHAP, an algorithm that uses an ANOVA-based functional decomposition model to approximate the black-box model being explained. This allows us to calculate Shapley values orders of magnitude faster than existing methods for large datasets, significantly reducing the amortized cost of computing Shapley values when many predictions need to be explained.

\end{abstract}
\section{Introduction}

Due to their theoretical foundations, Shapley values have gained significant popularity as an explanation method for black-box machine learning models in recent years \cite{shapley1997value,lundberg2018,Lundberg2017}. However, estimating Shapley values in practice is computationally very expensive, as the number of model evaluations needed to produce an exact explanation increases exponentially with the dimensionality of the data. This quickly becomes infeasible for high-dimensional data and/or computationally complex models.

Because of this problem, many techniques have been proposed to lower the cost of computing Shapley values. These techniques can roughly be divided into \textit{model-agnostic} and \textit{model-specific} techniques. Model-agnostic techniques make no assumptions about the black-box model being explained. These techniques rely on sampling subsets or permutations of features, and using those samples to estimate Shapley values. Examples include feature subset sampling \cite{strumbelj2014}, antithetic sampling \cite{mitchell} and KernelSHAP \cite{Lundberg2017}. In contrast, model-specific techniques exploit certain properties of the black-box model being explained to reduce the computational cost. For example, model-specific techniques have been proposed for ensembles of trees (TreeSHAP \cite{lundberg2019}) or neural network models (DeepSHAP \cite{Lundberg2017}). These techniques are often faster than the model-agnostic variants, but are less flexible.

In this work, we exploit properties of the functional ANOVA decomposition \cite{roosen1995} to produce a model-agnostic technique for approximating Shapley values with a very low \textit{amortized} cost: if many predictions need to be explained, the cost \textit{per explanation} decreases significantly. We do this by constructing a functional decomposition (ANOVA model), and training it to imitate the black box model being explained. Once the ANOVA model is trained, Shapley values can be estimated orders of magnitude faster than using existing model-agnostic approaches. We empirically show that the cost of training the surrogate model is compensated by the speedup in inference, even for relatively small amounts of explanations.

Recently, a related technique has been proposed that also trains a surrogate model to produce Shapley value explanations \cite{jethani2022}. However, this technique uses a black-box surrogate model to output Shapley values directly. This limits the possible uses of the resulting surrogate model. In contrast, our technique uses a functional decomposition model that retains information about the specific interactions between variables. This allows us not only to quickly compute Shapley values for a given feature, but also to view which other features each feature interacts with, an ability that is lost when using a black-box surrogate model \cite{lundberg2019a,sundararajan2020}.

\section{Shapley values}

We will first introduce some notation. We denote the black-box model as a function $f: \mathcal{X} \rightarrow \mathbb{R}$, where $\mathcal{X} \subseteq \mathbb{R}^d$. We abbreviate the set $\{1, \dots, d\}$ to $[d]$, and for a subset $u \subseteq [d]$ we write $-u := [d] \setminus u$. If $\mathbf{x} = (x_1, \dots, x_d) \in \mathcal{X}$ and $u \subseteq [d]$, we will write $\mathbf{x}_u := (x_j)_{j \in u}$, and analogously $\mathcal{X}_u := (\mathcal{X}_j)_{j \in u}$. We will also abbreviate $u \cup \{j\}$ to $u + j$, for any singleton set $\{j\}$. If $\mathbf{x}, \mathbf{y} \in \mathcal{X}$ and $u \subseteq [d]$, then $\mathbf{z} := \mathbf{x}_{u}:\mathbf{y}_{-u}$ is defined as $z_j := x_j$ for $j \in u$, and $z_j = y_j$ for $j \notin u$.

Originally introduced in the context of cooperative game theory, Shapley values \cite{shapley1997value} are a way of fairly distributing a \textit{payout} among participating \textit{players}. Let $\text{val}(u) \in \mathbb{R}$ be the payout of a subset of players $u \subseteq [d]$, with $\text{val}(\emptyset) = 0$. The Shapley value for player $j \in [d]$ is then defined as:
\begin{equation*}
\phi_j := \frac{1}{d} \sum_{u \subseteq -\{j\}} {d - 1 \choose |u|}^{-1}(val(u + j) - val(u))
\end{equation*}

The Shapley value is generally viewed as a ``fair'' way of attributing the payout to the players, because it adheres to the following properties \cite{shapley1997value,Lundberg2017}:
\begin{itemize}
\item \textbf{Efficiency:} $\sum_{j=1}^d \phi_j = \text{val}([d])$
\item \textbf{Symmetry:} If $\forall u \subseteq [d] \setminus \{i,j\}: \text{val}(u+j) = \text{val}(u+i)$, then $\phi_i = \phi_j$.
\item \textbf{Dummy:} If $\forall u \subseteq [d]: \text{val}(u + i) = \text{val}(u)$, then $\phi_i = 0$.
\item \textbf{Additivity:} If $\text{val}$ and $\text{val'}$ have Shapley values $\phi_j$ and $\phi'_j$ respectively, then the Shapley values for the game with value function $\text{val} + \text{val'}$ are $\phi_j + \phi'_j, \forall j \in [d]$.
\end{itemize}

\section{ANOVA and Partial Dependence Decomposition} 

The functional ANOVA is a decomposition of the form
\begin{equation*}
f(\mathbf{x}) = \sum_{u \subseteq [d]} f_u(\mathbf{x})
\end{equation*}
where each function $f_u$ depends only on the variables $x_j, j \in u$. The ANOVA decomposition is defined recursively as follows \cite{roosen1995}:
\begin{align}
f_\emptyset &:= \mu = \int f(\mathbf{x}) d\mathbf{x} \nonumber\\
f_u(\mathbf{x}) &:= \int_{\mathcal{X}_{-u}}\left( f(\mathbf{x}) - \sum_{v \subset u} f_v(\mathbf{x})\right) d\mathbf{x}_{-u} \nonumber\\
&= \sum_{v \subseteq u} (-1)^{|u| - |v|} \int_{\mathcal{X}_{-v}} f_v(\mathbf{x}) d\mathbf{x}_{-v}\nonumber
\end{align}
The ANOVA decomposition can also be viewed as decomposing the variance of the function $f$. If we denote the variance of $f$ and $f_u$ as:
\begin{align*}
\sigma^2 &:= \int_{\mathcal{X}_{-u}} (f(\mathbf{x}) - u)^2 d\mathbf{x}\\
\sigma^2_u &:= \left\{ \begin{array}{ll} \int f_u(\mathbf{x})^2 d\mathbf{x} &\qquad u \neq \emptyset \\ 0 &\qquad u = \emptyset \end{array} \right.
\end{align*}
then we have the following property:
\begin{equation*}
\sigma^2 = \sum_{|u| > 0} \sigma^2_u
\end{equation*}
This allows us to estimate the strength of each interaction $u$ using the variance of the corresponding function. For more information on the functional ANOVA decomposition, see \cite{roosen1995}.

The original ANOVA decomposition assumes that variables are independent and uniformly distributed: $x_j \sim \mathcal{U}[0,1], \forall j \in [d]$. This assumption is not critical, and the uniform distributions can be replaced with a general marginal distribution for each variable without loss of the decomposition of variance property \cite{hooker2004}: $x_j \sim X_j, \forall j \in [d]$. The decomposition then becomes:

\begin{align}
f_\emptyset &:= \mu = \int f(\mathbf{x}) d\mathbb{P}(\mathbf{x}) \nonumber\\
f_u(\mathbf{x}) &:= \int_{\mathcal{X}_{-u}}\left( f(\mathbf{x}) - \sum_{v \subset u} f_v(\mathbf{x})\right) d\mathbb{P}(\mathbf{x}_{-u}) \label{eq:anova_original}\\
&= \sum_{v \subseteq u} (-1)^{|u| - |v|} \int_{\mathcal{X}_{-v}} f_v(\mathbf{x}) d\mathbb{P}(\mathbf{x}_{-v}) \label{eq:anova_mobius}
\end{align}

where $\mathbb{P}(\mathbf{x}_u)$ denotes the marginal probability measure for $\mathcal{X}_u$. Note that this approach still assumes independence between variables.

The ANOVA decomposition with general marginal distributions corresponds to the Partial Dependence plot for univariate and bivariate effects \cite{friedman2001}. For this reason, we name this variant of the ANOVA decomposition the Partial Dependence Decomposition.

\section{Computing Shapley values from the Partial Dependence Decomposition}
\label{sec:computing-shapley-values-pdd}

\cite{owen2014} shows that, if we have a functional decomposition $f = \sum_u f_u$ and the value function is defined as:
\begin{equation*}
\text{val}(u) = \sum_{v \subseteq u} \sigma^2_v
\end{equation*}
then the Shapley values for this value function can be computed as follows:
\begin{equation*}
\phi_j = \sum_{\substack{u \subseteq [d] \\ j \in u}} \frac{\sigma^2_u}{|u|}
\end{equation*}
This specific implementation of Shapley values is termed \textit{Shapley Effects} \cite{song2016}.
The proof from \cite{owen2014} does not rely on any specific properties of the variance $\sigma^2_u$, and can therefore be used to prove the following, more general property of functional decompositions and Shapley values:

\begin{theorem}
	Let the value of a subset of variables be $\text{val}(u) := \sum_{v \subseteq u} g(f_v)$, where $\sum_{u \subseteq [d]} f_v = f$ and $g(f_u) \in \mathbb{R}, \forall u \subseteq [d]$. Then the Shapley value for variable $j$ is
	\begin{equation*}
	\phi_j = \sum_{\substack{u \subseteq [d]\\j \in u}} \frac{g(f_u)}{|u|}
	\end{equation*}
	\label{thm:decomp-shapley-property}
\end{theorem}
\begin{proof}
	We can write $\text{val}(u) = \sum_{v \subseteq [d], v \neq \emptyset} \text{val}^{(v)}(u)$, where $\text{val}^{(v)}(u) := g(f_v) \mathbbm{1}_{v \subseteq u}$ and $\mathbbm{1}_{v \subseteq u} = 1$ if $v \subseteq u$ and $0$ otherwise. Denote $\phi_j^{(v)}$ as the Shapley value for variable $j$ and value function $\text{val}^{(v)}$. Because of the Efficiency property, we know that $\sum_{j \in [d]} \phi^{(v)}_j = \text{val}^{(v)}([d]) = g(f_v)$. To compute $\phi_j^{(v)}$, we consider two cases:
	\begin{itemize}
		\item $j \notin v$: in this case, we have $\forall u \subseteq [d]: \text{val}^{(v)}(u + j) = g(f_v) \mathbbm{1}_{v \subseteq u + j} = g(f_v) \mathbbm{1}_{v \subseteq u} = \text{val}^{(v)}(u)$. Due to the Dummy property of Shapley values, we get $\forall j \notin v: \phi^{(v)}_j = 0$.
		\item $j \in v$: consider $i,j \in v, i \neq j, u \subseteq [d] \setminus \{i,j\}$. Then $\text{val}^{(v)}(u + i) = \text{val}^{(v)}(u + j) = 0$. Due to the Symmetry property, we know that $\forall i,j \in v: \phi_i = \phi_j$. Because $\sum_{j \in [d]} \phi^{(v)}_j = g(f_v)$ and $\forall j \notin v: \phi^{(v)}_j = 0$, we can conclude that $\forall j \in v: \phi_j = \frac{g(f_v)}{|v|}$.
	\end{itemize}
	Because of the Additivity property, we now have:
	\begin{equation*}
	\phi_j = \sum_{v \subseteq [d], v \neq \emptyset} \phi_j^{(v)} = \sum_{v \subseteq [d], j \in v} \frac{\sigma^2_v}{|v|}
	\end{equation*} $\qed$
\end{proof}

Note that the theorem holds for \textit{any} decomposition of the function $f$, not just the ANOVA decomposition.

We can now use this property to compute Shapley values from the Partial Dependence Decomposition. We define the value function for an explanation for an input point $\mathbf{x}$ as follows:
\begin{equation*}
\text{val}_\mathbf{x}(u) := \mathbb{E}_{\mathbf{z} \sim \mathcal{D}}[f(\mathbf{x}_u:\mathbf{z}_{-u})] - \mathbb{E}_{\mathbf{z} \sim \mathcal{D}}[f(\mathbf{z})]
\end{equation*}
where $\mathcal{D}$ is the input distribution. Using (\ref{eq:anova_mobius}), we can rewrite this as:
\begin{align*}
\text{val}_\mathbf{x}(u) &= \int_{\mathcal{X}_{-u}} f(\mathbf{x}_u:\mathbf{z}_{-u})d\mathbb{P}(\mathbf{z}_{-u}) - \int_\mathcal{X} f(\mathbf{z}) d\mathbb{P}(\mathbf{z})\\
&= \sum_{v \subseteq u} f_v(\mathbf{x}) - f_\emptyset
\end{align*}
If we now define $g(f_u)(\mathbf{x})$ as follows:
\begin{equation*}
g(f_u)(\mathbf{x}) = \left\{\begin{array}{ll}
0 & \mbox{if } u = \emptyset\\
f_u(\mathbf{x}) & \mbox{if } u \neq \emptyset
\end{array}\right.
\end{equation*}
then from Theorem \ref{thm:decomp-shapley-property}, we get:
\begin{align}
\text{val}_{\mathbf{x}}(u) &= \sum_{v \subseteq u} g(f_v)(\mathbf{x}) \nonumber\\
\phi^\mathbf{x}_j &= \sum_{\substack{u \subseteq [d]\\j \in u}} \frac{f_u(\mathbf{x})}{|u|} \label{eq:pdd-shapley}
\end{align}
where $\phi_j^\mathbf{x}$ is the Shapley value for variable $j$ and input sample $\mathbf{x}$.
\section{PDD-SHAP}
In Section \ref{sec:computing-shapley-values-pdd}, we showed that if we have access to a Partial Dependence Decomposition of the black-box function $f$, we can easily compute Shapley values for each variable $j$ without having to sample the input space or evaluate the original black-box model. The approach in PDD-SHAP is now as follows: we explicitly train a model $\hat{f}_u$ to approximate each $f_u$ of the Partial Dependence Decomposition, and then use that model to compute Shapley values using (\ref{eq:pdd-shapley}). We will denote $\hat{f} := \sum_{u \subseteq [d]} \hat{f}_u$ as the \textit{PDD surrogate model}.

A problem with this approach is given by the number of components $f_u$ that need to be estimated: this is equal to the number of subsets of $[d]$, which grows exponentially with $d$. However, we can mitigate this problem by observing that the higher-order terms in (\ref{eq:pdd-shapley}) have a lower impact on the result, as they are divided by their cardinality $|u|$. Also, we can reasonably assume that, in most real-world datasets, interactions between variables are generally of relatively low order \cite{hooker2004,owen2003}. From these observations, we can approximate the true PDD surrogate model using only the terms $f_u$ with $|u| \leq k$, for some $k \in \mathbb{N}$.

We can now summarize PDD-SHAP as follows:
\begin{itemize}
	\item For each subset $u \subseteq [d]$ with $|u| < k$:
	\begin{itemize}
		\item Approximate the partial dependence of $f$ on $u$ in a set of points $\mathbf{x}_u$ using (\ref{eq:anova_original})
		\item Train a model $\hat{f}(u)$ to approximate the partial dependence of $f$ on $u$
	\end{itemize}
	\item Compute Shapley values using (\ref{eq:pdd-shapley}).
\end{itemize}
The complete algorithm is given in Algorithm \ref{alg:pdd-shapley}. The background sample $X_{bg}$ is a representative subset of the original training data used to estimate the partial dependence. This background sample is also used in existing techniques for approximating Shapley values, and can be seen as the ``training data'' for the PDD surrogate model. We approximate the integral in (\ref{eq:anova_original}) using an empirical average on this background sample.

\begin{algorithm}
	\caption{PDD-SHAP training}\label{alg:pdd-shapley}
	\hspace*{\algorithmicindent} \textbf{Input:} $k \in \mathbb{N}$, background sample $X_{bg}$ with $|X_{bg}| = N$\\
	\hspace*{\algorithmicindent} \textbf{Output:} PDD surrogate model $\hat{f} := \{\hat{f}_u:u \subseteq [d], |u| \leq k\}$
	\begin{algorithmic}[1]
		\State $f_\emptyset \gets \frac{1}{N}\sum_{\mathbf{x} \in X_{bg}}f(\mathbf{x})$
		\For{$i = 1, \dots, k$}
		\For{$u \subseteq [d], |u| = i$}
		\For{$\mathbf{x}^j \in X_{bg}$}
		\State $y^j \gets \frac{1}{N}\sum_{z \in X_{bg}} f(\mathbf{x}_u:\mathbf{z}_{-u}) - \sum_{v \subset u} \hat{f}_v(\mathbf{x})$
		\EndFor
		\State Train model $\hat{f}_u$ on $\{(\mathbf{x}^j, y^j): \mathbf{x}^j \in X_{bg}\}$
		\EndFor
		\EndFor
		\State \Return $\{\hat{f}_u: u \subseteq [d], |u| \leq k\}$
	\end{algorithmic}
\end{algorithm}

\section{Experiments}
We tested our approach on the adult \cite{kohavi1996}, superconduct \cite{hamidieh2018}, UCI German credit \cite{Dua:2019}, California housing \cite{kelleypace1997} and UCI Abalone datasets \cite{Dua:2019}, using a background sample of 100 instances, and using a regression tree to model each $\hat{f}_u$\footnote{Implementation available at \url{https://github.com/arnegevaert/pdp-shapley}}. For the adult, superconduct, credit and California housing datasets, we compare the performance for $k \in \{1, \dots, 4\}$. For the superconduct dataset, we only compare the performance for $k \in \{1,2\}$, as this dataset has a significantly larger number of features than the others. An overview of the datasets is given in Table \ref{tab-datasets}.

To evaluate our technique, we train a Gradient Boosting model on each dataset and compare the explanations given by PDD-SHAP for 1000 test samples to 3 existing model-agnostic approaches for computing Shapley values: feature subset sampling \cite{strumbeljerik2010}, antithetic sampling \cite{mitchell} and KernelSHAP \cite{Lundberg2017} (note that for the German credit and Abalone datasets, we used the full test set, as they contained fewer than 1000 samples). These approaches are all implemented in the \texttt{shap} package\footnote{\url{https://github.com/slundberg/shap}} (\texttt{SamplingExplainer}, \texttt{PermutationExplainer} and \texttt{KernelExplainer}, respectively).

Figures \ref{fig:acc-r2} and \ref{fig:acc-corr} resp. show the $R^2$ and Spearman rank correlations between the different implementations of PDD-SHAP and the subset sampling approach. We see that for most datasets, the $R^2$ value is close to 0.9 or larger for $k=2$, indicating that the model contains almost no interactions between more than 2 features. On the abalone dataset, we see the $R^2$ still increases for $k > 2$, indicating that this dataset contains higher-order features. In general, the Spearman rank correlations are higher than the $R^2$ values. This indicates that, although the absolute magnitude of each Shapley value may not be entirely accurate, the \textit{ranking} of the features according to their importance is accurately estimated in most cases.

Table \ref{tab-speed} shows the runtime required for each method to generate Shapley values on 1000 instances. We see that the inference time for PDD-SHAP is orders of magnitude lower than the runtime for the existing methods. In many cases even the runtime for training the surrogate model and inference combined is still significantly lower than the runtime for the existing methods.

\begin{table}
	\centering
	\caption{Overview of the datasets.}\label{tab-datasets}
	\begin{tabular}{|l|l|l|l|}
		\hline
		Name &  Type & \#features & \#samples\\
		\hline
		Adult & Classification & 15 & 48842\\
		Credit & Classification & 20 & 1000\\
		Superconduct & Regression & 82 & 21263\\
		Housing & Regression & 10 & 20640\\
		Abalone & Classification & 9 & 4177\\
		\hline
	\end{tabular}
\end{table}

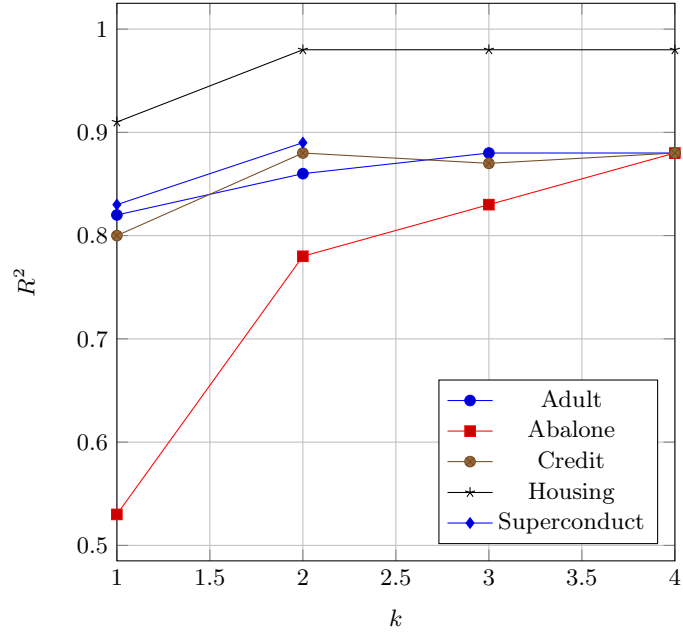
\begin{figure}
	\centering
	\begin{tikzpicture}
	\begin{axis}[
	height=9cm,
	width=9cm,
	grid=major,
	xlabel=$k$,
	ylabel=$R^2$,
	xmin=1,xmax=4,
	legend pos=south east
	]
	
	\addplot coordinates {
		(1,0.82)
		(2,0.86)
		(3,0.88)
		(4,0.88)
	};
	\addlegendentry{Adult}
	
	\addplot coordinates {
		(1,0.53)
		(2,0.78)
		(3,0.83)
		(4,0.88)
	};
	\addlegendentry{Abalone}
	
	\addplot coordinates {
		(1,0.80)
		(2,0.88)
		(3,0.87)
		(4,0.88)
	};
	\addlegendentry{Credit}
	
	\addplot coordinates {
		(1,0.91)
		(2,0.98)
		(3,0.98)
		(4,0.98)
	};
	\addlegendentry{Housing}
	
	\addplot coordinates {
		(1,0.83)
		(2,0.89)
	};
	\addlegendentry{Superconduct}
	\end{axis}
	\end{tikzpicture}
	\caption{Accuracy of PDD-SHAP in terms of $R^2$. Y-axis shows $R^2$ value between PDD-SHAP and feature subset sampling, X-axis shows the value for $k$ (the maximal order of interaction in PDD-SHAP).}
	\label{fig:acc-r2}
\end{figure}

\begin{figure}
	\centering
	\begin{tikzpicture}
	\begin{axis}[
	height=9cm,
	width=9cm,
	grid=major,
	xlabel=$k$,
	ylabel=$\rho$,
	xmin=1,xmax=4,
	legend pos=south east
	]
	
	\addplot coordinates {
		(1,0.93)
		(2,0.93)
		(3,0.93)
		(4,0.93)
	};
	\addlegendentry{Adult}
	
	\addplot coordinates {
		(1,0.70)
		(2,0.81)
		(3,0.86)
		(4,0.87)
	};
	\addlegendentry{Abalone}
	
	\addplot coordinates {
		(1,0.91)
		(2,0.94)
		(3,0.93)
		(4,0.94)
	};
	\addlegendentry{Credit}
	
	\addplot coordinates {
		(1,0.96)
		(2,0.97)
		(3,0.97)
		(4,0.97)
	};
	\addlegendentry{Housing}
	
	\addplot coordinates {
		(1,0.81)
		(2,0.84)
	};
	\addlegendentry{Superconduct}
	\end{axis}
	\end{tikzpicture}
	\caption{Accuracy of PDD-SHAP in terms of Spearman rank correlation. Y-axis shows the spearman correlation $\rho$ between PDD-SHAP and feature subset sampling, X-axis shows the value for $k$ (the maximal order of interaction in PDD-SHAP).}
	\label{fig:acc-corr}
\end{figure}
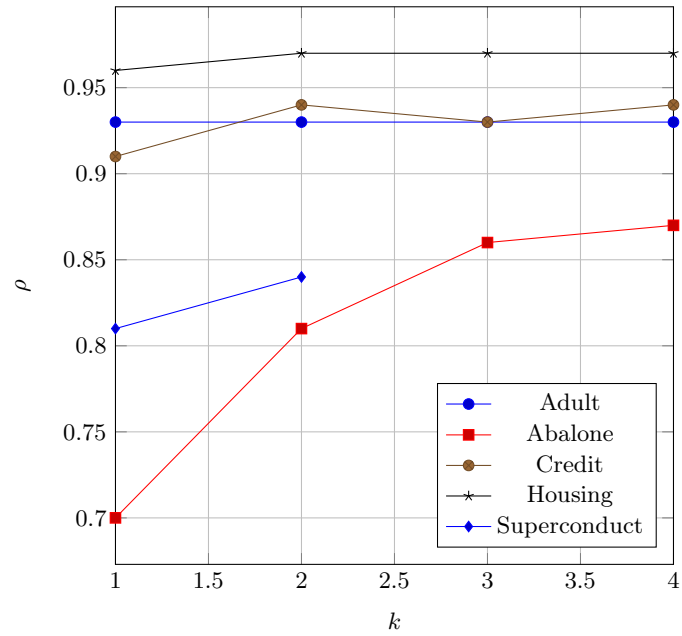

\begin{table}
	\caption{Runtime for PDD-SHAP vs. subset sampling, antithetic sampling and KernelSHAP for 1000 explanations, with a background dataset of 100 intances. The rightmost columns correspond to PDD-SHAP for varying values of $k$. Times are given as (train)+(inference). Results are indicated in bold where the sum of training and inference time for PDD-SHAP is lower than the runtime of all 3 alternatives.}\label{tab-speed}
	\begin{center}
	\begin{tabular}{|l|c|c|c||c|c|c|c|}
		\hline
		\multirow{2}{*}{Dataset} & \multirow{2}{*}{\shortstack{Subset\\sampling}} & \multirow{2}{*}{\shortstack{Antithetic\\sampling}} & \multirow{2}{*}{\shortstack{Kernel-\\SHAP}} & \multicolumn{4}{c|}{PDD-SHAP (train time + inference time)} \\
		\cline{5-8}
		& & & & $k=1$ & $k=2$ & $k=3$ & $k=4$ \\
		
		\hline
		Adult & 103.90 & 67.44 & 1060.78 & \textbf{0.85+0.01} & \textbf{6.15+0.02} & \textbf{27.33+0.10} & 84.29+0.28 \\
		Credit &  112.16 & 73.09 & 1125.11 & \textbf{1.31+0.01} & \textbf{12.42+0.04} & 77.96+0.22 & 370.31+1.06 \\
		Superconduct & 308.30 & 97.42 & 2141.26 & \textbf{2.98+0.01} & 120.09+0.18 & N/A & N/A \\
		Housing & 68.05 & 89.25 & 239.31 & \textbf{0.43+0.01} & \textbf{1.64+0.01} & \textbf{4.14+0.02} & \textbf{8.03+0.03} \\
		Abalone & 108.66 & 230.72 & 169.25 & \textbf{1.08+0.01} & \textbf{4.63+0.02} & 11.61+0.03 & 18.80+0.05 \\
		\hline
	\end{tabular}
	\end{center}
\end{table}

\section{Conclusion and future work}
In this work, we introduce PDD-SHAP, a model-agnostic approach to approximating Shapley value explanations using a surrogate functional decomposition model. We test our approach on a selection of datasets, and show that the approximations are accurate, while the amortized cost of computation is orders of magnitude lower than other model-agnostic alternatives. While these results are already promising, there are still many possibilities to improve the speed and accuracy of the algorithm. Some approaches that we intend to explore include a more intelligent selection of higher-order effects in the ANOVA model, Monte Carlo integration techniques for the estimation of partial dependence, and more efficient background distribution sampling. 

\subsubsection{Acknowledgements} The research leading to these results has received funding from the Flemish Government under the ``Onderzoeksprogramma Artifici\"ele Intelligentie (AI) Vlaanderen'' programme, and from the BOF project 01D13919.

%
%
%
\bibliographystyle{splncs04}
\bibliography{bibliography}

\begin{thebibliography}{10}
\providecommand{\url}[1]{\texttt{#1}}
\providecommand{\urlprefix}{URL }
\providecommand{\doi}[1]{https://doi.org/#1}

\bibitem{Dua:2019}
Dua, D., Graff, C.: {{UCI}} machine learning repository (2017)

\bibitem{friedman2001}
Friedman, J.H.: Greedy {{Function Approximation}}: {{A Gradient Boosting
  Machine}}. The Annals of Statistics  \textbf{29}(5),  1189--1232 (2001)

\bibitem{hamidieh2018}
Hamidieh, K.: A data-driven statistical model for predicting the critical
  temperature of a superconductor. Computational Materials Science
  \textbf{154},  346--354 (Nov 2018). \doi{10.1016/j.commatsci.2018.07.052}

\bibitem{hooker2004}
Hooker, G.: Discovering additive structure in black box functions. In:
  Proceedings of the 2004 {{ACM SIGKDD}} International Conference on
  {{Knowledge}} Discovery and Data Mining - {{KDD}} '04. p.~575. {ACM Press},
  {Seattle, WA, USA} (2004). \doi{10.1145/1014052.1014122}

\bibitem{jethani2022}
Jethani, N., Sudarshan, M., Covert, I., Lee, S.I., Ranganath, R.: {{FastSHAP}}:
  {{Real-Time Shapley Value Estimation}} (Mar 2022).
  \doi{10.48550/arXiv.2107.07436}

\bibitem{kelleypace1997}
Kelley~Pace, R., Barry, R.: Sparse spatial autoregressions. Statistics \&
  Probability Letters  \textbf{33}(3),  291--297 (May 1997).
  \doi{10.1016/S0167-7152(96)00140-X}

\bibitem{kohavi1996}
Kohavi, R.: Scaling {{Up}} the {{Accuracy}} of {{Naive-Bayes Classifiers}}: A
  {{Decision-Tree Hybrid}}. Proceedings of the Second International Conference
  on Knowledge Discovery and Data Mining p.~6 (1996)

\bibitem{Lundberg2017}
Lundberg, S., Lee, S.I.: A {{Unified Approach}} to {{Interpreting Model
  Predictions}}. Advances in Neural Information Processing Systems
  \textbf{30},  4766--4775 (2017)

\bibitem{lundberg2019}
Lundberg, S.M., Erion, G., Chen, H., DeGrave, A., Prutkin, J.M., Nair, B.,
  Katz, R., Himmelfarb, J., Bansal, N., Lee, S.I.: Explainable {{AI}} for
  {{Trees}}: {{From Local Explanations}} to {{Global Understanding}}.
  arXiv:1905.04610 [cs, stat]  (May 2019)

\bibitem{lundberg2019a}
Lundberg, S.M., Erion, G.G., Lee, S.I.: Consistent {{Individualized Feature
  Attribution}} for {{Tree Ensembles}} (Mar 2019).
  \doi{10.48550/arXiv.1802.03888}

\bibitem{lundberg2018}
Lundberg, S.M., Nair, B., Vavilala, M.S., Horibe, M., Eisses, M.J., Adams, T.,
  Liston, D.E., Low, D.K.W., Newman, S.F., Kim, J., Lee, S.I.: Explainable
  machine-learning predictions for the prevention of hypoxaemia during surgery.
  Nature Biomedical Engineering  \textbf{2}(10),  749--760 (Oct 2018).
  \doi{10.1038/s41551-018-0304-0}

\bibitem{mitchell}
Mitchell, R., Cooper, J., Frank, E., Holmes, G.: Sampling {{Permutations}} for
  {{Shapley Value Estimation}} p.~46

\bibitem{owen2003}
Owen, A.B.: The {{Dimension Distribution}} and {{Quadrature Test Functions}}.
  Statistica Sinica  \textbf{13}(1),  1--17 (2003)

\bibitem{owen2014}
Owen, A.B.: Sobol' {{Indices}} and {{Shapley Value}}. SIAM/ASA Journal on
  Uncertainty Quantification  \textbf{2}(1),  245--251 (Jan 2014).
  \doi{10.1137/130936233}

\bibitem{roosen1995}
Roosen, C.B., Friedman, J.H., Owen, A.B.: Visualization {{And Exploration Of
  High-Dimensional Functions Using The Functional Anova Decomposition}} (1995)

\bibitem{shapley1997value}
Shapley, L.S.: A value for n-person games. Contributions to the Theory of Games
   \textbf{2}(28),  307--317 (1953)

\bibitem{song2016}
Song, E., Nelson, B.L., Staum, J.: Shapley {{Effects}} for {{Global Sensitivity
  Analysis}}: {{Theory}} and {{Computation}}. SIAM/ASA Journal on Uncertainty
  Quantification  \textbf{4}(1),  1060--1083 (Jan 2016).
  \doi{10.1137/15M1048070}

\bibitem{strumbelj2014}
{\v S}trumbelj, E., Kononenko, I.: Explaining prediction models and individual
  predictions with feature contributions. Knowledge and Information Systems
  \textbf{41}(3),  647--665 (Dec 2014). \doi{10.1007/s10115-013-0679-x}

\bibitem{strumbeljerik2010}
{\v S}trumbelj, Erik, E., Kononenko, I.: An {{Efficient Explanation}} of
  {{Individual Classifications}} using {{Game Theory}}. Journal of Machine
  Learning Research  \textbf{11}(1),  1--18 (2010)

\bibitem{sundararajan2020}
Sundararajan, M., Dhamdhere, K., Agarwal, A.: The {{Shapley Taylor Interaction
  Index}}. In: Proceedings of the 37th {{International Conference}} on
  {{Machine Learning}}. pp. 9259--9268. {PMLR} (Nov 2020)

\end{thebibliography}
\end{document}